\documentclass[letterpaper]{article}
\usepackage{natbib}
\usepackage{times}
\usepackage{helvet}
\usepackage{courier}

\usepackage{url}
\usepackage{algorithm}
\usepackage{algorithmic}
\usepackage{graphicx}
\usepackage{epstopdf}
\usepackage{amsmath}
\usepackage{amsthm}
\usepackage{amsfonts}

\newcommand{\var}{\text{Var}}
\newcommand{\Bellman}{T}
\newcommand{\Pmat}{P}
\newcommand{\Rvec}{R}
\newcommand{\tpol}{\pi}
\newcommand{\bpol}{\mu}
\newcommand{\tdist}{d_\pi}
\newcommand{\bdist}{d_\mu}
\newcommand{\T}{^\top}
\newcommand{\R}{\mathbb R}

\newcommand{\ExpBy}[3]{\mathbb{E}^{#2}\left[ \left. #1 \right| #3 \right]}

\newcommand{\beq}{\begin{equation}}
\newcommand{\eeq}{\end{equation}}

\newcommand{\beqa}{\begin{eqnarray}}
\newcommand{\eeqa}{\end{eqnarray}}

\newcommand{\beqan}{\begin{eqnarray*}}
\newcommand{\eeqan}{\end{eqnarray*}}
\newcommand{\norm}[1]{\left\lVert#1\right\rVert}

\renewcommand{\P}{\mathbb{P}}

\newcommand{\E}{\mathbb{E}}

\renewcommand{\phi}{\varphi}
\renewcommand{\epsilon}{\varepsilon}

\newcommand{\eqdef}{\stackrel{\rm def}{=}}

\newcommand{\ETD}{ETD($\lambda$, $\beta$)}
\newcommand{\ETDO}{ETD($0$, $\beta$)}

\newtheorem{Lemma}{Lemma}
\newtheorem{Theorem}{Theorem}
\newtheorem{Example}{Example}

\newtheorem{Remark}{Remark}
\newtheorem{corollary}{Corollary}
\newtheorem{Proposition}{Proposition}

\begin{document}

%
\title{Generalized Emphatic Temporal Difference Learning: Bias-Variance Analysis}
\author{Assaf Hallak, Aviv Tamar, R$\acute{\text{e}}$mi Munos, Shie Mannor}

\maketitle
\begin{abstract}
\begin{quote}
We consider the off-policy evaluation problem in Markov decision processes with function approximation. We propose a generalization of the recently introduced \emph{emphatic temporal differences} (ETD) algorithm \citep{SuttonMW15}, which encompasses the original ETD($\lambda$), as well as several other off-policy evaluation algorithms as special cases. We call this framework \ETD, where our introduced parameter $\beta$ controls the decay rate of an importance-sampling term. We study conditions under which the projected fixed-point equation underlying \ETD\ involves a contraction operator, allowing us to present the first asymptotic error bounds (bias) for \ETD. Our results show that the original ETD algorithm always involves a contraction operator, and its bias is bounded. Moreover, by controlling $\beta$, our proposed generalization allows trading-off bias for variance reduction, thereby achieving a lower total error.
\end{quote}
\end{abstract}

\section{Introduction}
In Reinforcement Learning (RL; \citealt{sutton_reinforcement_1998}), \emph{policy-evaluation} refers to the problem of evaluating the value function -- a mapping from states to their long-term discounted return under a given policy, using sampled observations of the system dynamics and reward. Policy-evaluation is important both for assessing the quality of a policy, but also as a sub-procedure for policy optimization.

For systems with large or continuous state-spaces, an exact computation of the value function is often impossible. Instead, an \emph{approximate} value-function is sought using various function-approximation techniques (a.k.a. approximate dynamic-programming; \citealt{Ber2012DynamicProgramming}). In this approach, the parameters of the value-function approximation are tuned using machine-learning inspired methods, often based on \emph{temporal-differences} (TD;\citealt{sutton_reinforcement_1998}).

The source generating the sampled data divides policy evaluation into two cases. In the \emph{on-policy} case, the samples are generated by the \emph{target-policy} -- the policy under evaluation; In the \emph{off-policy} setting, a different \emph{behavior-policy} generates the data. In the on-policy setting, TD methods are well understood, with classic convergence guarantees and approximation-error bounds, based on a contraction property of the projected Bellman operator underlying TD \citep{BT96}. These bounds guarantee that the asymptotic error, or ~\emph{bias}, of the algorithm is contained. For the off-policy case, however, standard TD methods no longer maintain this contraction property, the error bounds do not hold, and these methods might even diverge \citep{baird1995residual}.

The standard error-bounds may be shown to hold for an \emph{importance-sampling} TD method (IS-TD), as proposed by \citet{precup2001off}. However, this method is known to suffer from a high variance of its importance-sampling estimator, limiting its practicality.

Lately, \citet{SuttonMW15} proposed the \emph{emphatic TD} (ETD) algorithm: a modification of the TD idea, which converges off-policy \citep{yu2015etd}, and has a reduced variance compared to IS-TD. This variance reduction is achieved by incorporating a certain decay factor over the importance-sampling ratio. However, to the best of our knowledge, there are no results that bound the bias of ETD. Thus, while ETD is assured to converge, it is not known how good its limit actually is.

In this paper, we propose the \ETD\ framework -- a modification of the ETD($\lambda$) algorithm, where the decay rate of the importance-sampling ratio, $\beta$, is a free parameter, and $\lambda$ is the same bootstrapping parameter employed in TD($\lambda$) and ETD($\lambda$). By varying the decay rate, one can smoothly transition between the IS-TD algorithm, through ETD, to the standard TD algorithm.

We investigate the bias of \ETD, by studying the conditions under which its underlying projected Bellman operator is a contraction. We show that the original ETD possesses a contraction property, and present the first error bounds for ETD and \ETD. In addition, our error bound reveals that the decay rate parameter balances between the bias and variance of the learning procedure. In particular, we show that selecting a decay equal to the discount factor as in the original ETD may be suboptimal in terms of the mean-squared error.

The main contributions of this work are therefore a unification of several off-policy TD algorithms under the \ETD\ framework, and a new error analysis that reveals the bias-variance trade-off between them.

\paragraph{Related Work:} In recent years, several different off-policy policy-evaluation algorithms have been studied, such as importance-sampling based least-squares TD \citep{yu2012least}, and gradient-based TD \citep{sutton2009fast,liu2015finite}. These algorithms are guaranteed to converge, however, their asymptotic error can be bounded only when the target and behavior policies are similar \citep{bertsekas2009projected}, or when their induced transition matrices satisfy a certain matrix-inequality suggested by \citet{kolter2011fixed}, which limits the discrepancy between the target and behavior policies. When these conditions are not satisfied, the error may be arbitrarily large \citep{kolter2011fixed}. In contrast, the approximation-error bounds in this paper hold for \emph{general target and behavior policies}.

\section{Preliminaries}
We consider an MDP $M=(S, A, P, R, \gamma)$, where $S$ is the state space, $A$ is the action space, $P$ is the transition probability matrix, $R$ is the reward function, and $\gamma\in [0,1)$ is the discount factor.

Given a target policy $\tpol$ mapping states to a distribution over actions, our goal is to evaluate the \emph{value function}:
\begin{equation*}
    V^\tpol(s) \doteq \ExpBy{\sum_{t=0}^\infty R(s_t,a_t)}{\tpol}{s_0 = s}.
\end{equation*}

Linear temporal difference methods \citep{sutton_reinforcement_1998} approximate the value function by
\begin{equation*}
    V^\tpol(s) \approx \theta \T \phi(s)  ,
\end{equation*}
where $\phi(s)\in \R^n$ are state features, and $\theta \in \R^n$ are weights, and use sampling to find a suitable $\theta$.
Let $\bpol$ denote a behavior policy that generates the samples $s_0,a_0,s_1,a_1,\dots$ according to $a_t \sim \bpol(\cdot|s_t)$ and $s_{t+1}\sim P(\cdot|s_t,a_t)$. We denote by $\rho_t$ the ratio $\tpol(a_t|s_t)/\bpol(a_t|s_t)$, and we assume, similarly to \citet{SuttonMW15}, that $\bpol$ and $\tpol$ are such that $\rho_t$ is well-defined\footnote{Namely, if $\bpol(a|s)=0$ then $\tpol(a|s)=0$ for all $s\in S$.} for all $t$.

Let $\Bellman$ denote the Bellman operator for policy $\tpol$, given by $$\Bellman (V) \doteq \Rvec + \gamma \Pmat V,$$ where $\Rvec$ and $\Pmat$ are the reward vector and transition matrix induced by policy $\tpol$, and let $\Phi$ denote a matrix whose columns are the feature vectors for all states. Let $\bdist$ and $\tdist$ denote the stationary distributions over states induced by the policies $\bpol$ and $\tpol$, respectively. For some $d\in \R^{|S|}$ satisfying $d>0$ element-wise, we denote by $\Pi_d$ a projection to the subspace spanned by $\phi(s)$ with respect to the $d$-weighted Euclidean-norm.

For $\lambda=0$, the \ETDO\ \citep{SuttonMW15} algorithm seeks to find a good approximation of the value function by iteratively updating the weight vector $\theta$:
\begin{equation}\label{eq:ETD}
\begin{split}
    \theta_{t+1} &= \theta_t + \alpha F_t \rho_t (R_{t+1} + \gamma \theta_t \T \phi_{t+1} - \theta_t \T \phi_t) \phi_t \\
    F_t &= \beta \rho_{t-1}F_{t-1} + 1, \quad F_0 = 1,
\end{split}
\end{equation}
where $F_t$ is a decaying trace of the importance-sampling ratios, and $\beta\in (0,1)$ controls the decay rate.
\begin{Remark}
The algorithm of \citet{SuttonMW15} selects the decay rate equal to the discount factor, i.e., $\beta = \gamma$. Here, we provide more freedom in choosing the decay rate. As our analysis reveals, the decay rate controls a bias-variance trade-off of ETD, therefore this freedom is important. Moreover, we note that for $\beta=0$, we obtain the standard TD in an off-policy setting \cite{yu2012least}, and when $\beta=1$ we obtain the full importance-sampling TD algorithm \cite{precup2001off}.
\end{Remark}
\begin{Remark}
The ETD($0$, $\gamma$) algorithm of \citet{SuttonMW15} also includes a state-dependent emphasis weight $i(s)$, and a state-dependent discount factor $\gamma(s)$. Here, we analyze the case of a uniform weight $i(s)=1$ and constant discount factor $\gamma$ for all states. While our analysis can be extended to their more general setting, the insights from the analysis remain the same, and for the purpose of clarity we chose to focus on this simpler setting.
\end{Remark}

An important term in our analysis is the emphatic weight vector $f$, defined by
\begin{equation}\label{eq:f}
    f \T = \bdist \T (I - \beta \Pmat)^{-1}.
\end{equation}
It can be shown \citep{SuttonMW15,yu2015etd}, that \ETDO\ converges to $\theta^*$ - a solution of the following \emph{projected fixed point equation}:
\begin{equation}\label{eq:fixed_point_eq}
    V = \Pi_f \Bellman V, \qquad V\in \R^{|S|}.
\end{equation}
For the fixed point equation \eqref{eq:fixed_point_eq}, a contraction property of $\Pi_f \Bellman$ is important for guaranteeing both a unique solution, and a bias bound \citep{BT96}.

It is well known that $\Bellman$ is a $\gamma$-contraction with respect to the $\tdist$-weighted Euclidean norm \citep{BT96}, and by definition $\Pi_f$ is a non-expansion in $f$-norm, however, it is not immediate that the composed operator $\Pi_f \Bellman$ is a contraction in any norm. Indeed, for the TD(0) algorithm (\citealt{sutton_reinforcement_1998}; corresponding to the $\beta=0$ case in our setting), a similar representation as a projected Bellman operator holds, but it may be shown that in the off-policy setting the algorithm might diverge \citep{baird1995residual}.
In the next section, we study the contraction properties of $\Pi_f \Bellman$, and provide corresponding bias bounds.

\section{Bias of ETD($0$, $\beta$)}

%
In this section we study the bias of the \ETDO\ algorithm.
Let us first introduce the following measure of discrepancy between the target and behavior policies:
\begin{equation*}
\kappa \doteq \min_s \frac{\bdist (s)}{f(s)}.
\end{equation*}
\begin{Lemma}\label{lemma:kappa_bounds}
The measure $\kappa$ obtains values ranging from $\kappa = 0$ (when there is a state visited by the target policy, but not the behavior policy), to $\kappa = 1-\beta$ (when the two policies are identical).
\end{Lemma}
The technical proof is given in the supplementary material. The following theorem shows that for \ETDO\ with a suitable $\beta$, the projected Bellman operator $\Pi_f \Bellman$ is indeed a contraction.
\begin{Theorem}\label{thm:one}
For $\beta>\gamma^2(1 - \kappa)$, the projected Bellman operator $\Pi_f \Bellman$ is a $\sqrt{\frac{\gamma^2}{\beta} (1 - \kappa)}$-contraction with respect to the Euclidean $f$-weighted norm, namely, $\forall v_1,v_2\in \R^{|S|}$:
\begin{equation*}
    \norm{ \Pi_f \Bellman v_1 - \Pi_f \Bellman v_2 }_f \leq \sqrt{\frac{\gamma^2}{\beta}(1-\kappa)} \|v_1 - v_2\|_f.
\end{equation*}
\end{Theorem}

\begin{proof}
Let $F = diag(f)$. We have
\begin{equation*}
\begin{split}
\norm{ v }^2_f - \beta \norm{ \Pmat v }^2_f &= v^\top F v - \beta v^\top \Pmat^\top F \Pmat v  \\
&\geq^{(a)} v^\top F v - \beta v^\top diag(f^\top \Pmat)v \\
& = v^\top [F - \beta diag(f^\top \Pmat)]v \\
& = v^\top \left[diag \left(f^\top (I-\beta \Pmat) \right)  \right] v \\
&=^{(b)} v^\top diag(\bdist) v = \norm{ v }^2_{\bdist},
\end{split}
\end{equation*}
where (a) follows from Jensen inequality:
\begin{equation*}
\begin{split}
v^\top \Pmat^\top F \Pmat v &= \sum_s f(s) ( \sum_{s'} \Pmat(s'|s) v(s') )^2 \\
&\leq \sum_s f(s)  \sum_{s'}\Pmat(s'|s)  v^2(s') \\
&= \sum_{s'} v^2(s') \sum_s f(s) \Pmat(s'|s) \\
&= v^\top diag(f^\top \Pmat)v,
\end{split}
\end{equation*}

and (b) is by the definition of $f$ in \eqref{eq:f}.

Notice that for every $v$:
\begin{equation*}
\norm{ v }^2_{\bdist} = \sum_s \bdist(s) v^2(s)  \geq \sum_s \kappa f(s) v^2(s)   = \kappa \norm{ v }^2_f
\end{equation*}

Therefore:
\begin{equation*}
\begin{split}
\norm{ v }^2_f & \geq \beta \norm{\Pmat v } ^2_f + \norm{ v }^2_{\bdist} \geq \beta \norm{\Pmat v } ^2_f + \kappa \norm{ v }^2_f, \\
\Rightarrow & \quad \beta \norm{\Pmat v}^2_f \leq (1-\kappa) \norm{ v }^2_f
\end{split}
\end{equation*}
and:
\begin{equation*}
\begin{split}
  \norm{\Bellman v_1 - \Bellman v_2 }^2_f  &= \norm{ \gamma \Pmat(v_1 - v_2) } ^2_f \\
  & = \gamma^2 \norm{ \Pmat (v_1 - v_2) }^2_f \\
  &\leq \frac{\gamma^2}{\beta} (1-\kappa) \norm{ v_1 - v_2 } ^2_f .
\end{split}
\end{equation*}
Hence, $T$ is a $\sqrt{\frac{\gamma^2}{\beta}(1-\kappa)}$-contraction. Since $\Pi_f$ is a non-expansion in the $f$-weighted norm \citep{BT96}, $\Pi_f T$ is a $\sqrt{\frac{\gamma^2}{\beta}(1-\kappa)}$-contraction as well.
\end{proof}

Recall that for the original ETD algorithm \citep{SuttonMW15}, we have that $\beta=\gamma$, and the contraction modulus is $\sqrt{\gamma(1-\kappa)}<1$, thus the contraction of $\Pi_f T$ always holds.

Also note that in the on-policy case, the behavior and target policies are equal, and according to Lemma \ref{lemma:kappa_bounds} we have $1-\kappa = \beta$. In this case, the contraction modulus in Theorem \ref{thm:one} is $\gamma$, similar to the result for on-policy TD \cite{BT96}.

We remark that \citet{kolter2011fixed} also used a measure of discrepancy between the behavior and the target policy to bound the TD-error. However, \citet{kolter2011fixed} considered the standard TD algorithm, for which a contraction could be guaranteed only for a class of behavior policies that satisfy a certain matrix inequality criterion. Our results show that for \ETDO\ with a suitable $\beta$, a contraction is guaranteed for \emph{general} behavior policies. We now show in an example that our contraction modulus bounds are tight.
\begin{Example}\label{example:tightness}
Consider an MDP with two states: Left and Right. In each state there are two identical actions leading to either Left or Right deterministically. The behavior policy will choose Right with probability $\epsilon$, and the target policy will choose Left with probability $\epsilon$, hence $1-\kappa \approx 1$. Calculating the quantities of interest:
\begin{equation*}
\begin{split}
\Pmat = \left( \begin{array}{cc}
\epsilon & 1-\epsilon  \\
\epsilon & 1-\epsilon
\end{array} \right)
, \quad \bdist = \left( 1-\epsilon, \epsilon \right) \\
f = \frac{1}{1-\beta} \left( 1 + 2 \epsilon \beta - \epsilon - \beta , -2 \epsilon \beta + \epsilon + \beta \right) \T.
\end{split}
\end{equation*}
So for $v = \left( 0, 1 \right) \T$:
\begin{equation*}
\norm{ v }^2_f = \frac{\epsilon + \beta  - 2\epsilon\beta}{1-\beta}, \quad \norm{ \Pmat v }^2_f = \frac{ (1-\epsilon)^2 }{1-\beta},
\end{equation*}
and for small $\epsilon$ we obtain that $\frac{\norm{ \gamma \Pmat v }^2}{\norm{ v }^2_f} \approx \frac{\gamma^2}{\beta} $.
\end{Example}

An immediate consequence of Theorem \ref{thm:one} is the following error bound, based on Lemma 6.9 of \citet{BT96}:
\begin{corollary}\label{corr:err}
We have
\begin{equation*}
\begin{split}
    \norm{ \Phi \T \theta^* - V^\tpol }_f & \leq \frac{1}{\sqrt{1 - \frac{\gamma^2}{\beta}(1-\kappa)}} \norm{ \Pi_f V^\tpol - V^\tpol }_f, \\
    \norm{ \Phi \T \theta^* - V^\tpol }_{d_\mu} &\leq \frac{1}{\sqrt{\gamma \left(1 - \frac{\gamma^2}{\beta}(1-\kappa) \right)}} \norm{ \Pi_f V^\tpol - V^\tpol }_f.
    \end{split}
\end{equation*}
\end{corollary}
Up to the weights in the norm, the error $\norm{ \Pi_f V^\tpol - V^\tpol }_f$ is the best approximation we can hope for, within the capability of the linear approximation architecture. Corollary \ref{corr:err} guarantees that we are not too far away from it.

Notice that the error $\norm{ \Phi \T \theta^* - V^\tpol }_{d_\mu}$ uses a measure $d_\mu$ which is independent of the target policy; This could be useful in further analysis of a policy iteration algorithm, which iteratively improves the target policy using samples from a single behavior policy. Such an analysis may proceed similarly to that in \citet{munos2003error} for the on-policy case.

\subsection{Numerical Illustration}\label{ssec:example_Kolter1}
We illustrate the importance of the \ETDO\ bias bound in a numerical example. Consider the 2-state MDP example of \citet{kolter2011fixed}, with transition matrix $P=(1/2) \underline{\underline{\textbf{1}}}$ (where $\underline{\underline{\textbf{1}}}$ is an all $1$ matrix), discount factor $\gamma=0.99$, and value function $V=[1, 1.05]\T$ (with $R=(I-\gamma P)V$). The features are $\Phi = [1, 1.05+\epsilon]\T$, with $\epsilon=0.001$. Clearly, in this example we have $\tdist=[0.5,0.5]$. The behavior policy is chosen such that $\bdist = [p, 1-p]$.

In Figure \ref{fig:off_policy_bias} we plot the mean-squared error $\norm{ \Phi \T \theta^* - V^\tpol }_{\tdist}$, where $\theta^*$ is either the fixed point of the standard TD equation $V = \Pi_{\bdist} \Bellman V$, or the \ETDO\ fixed point of \eqref{eq:fixed_point_eq}, with $\beta=\gamma$. We also show the optimal error $\norm{ \Pi_{\tdist}V - V^\tpol }_{\tdist}$ achievable with these features. Note that, as observed by \citet{kolter2011fixed}, for certain behavior policies the bias of standard TD is infinite. This means that algorithms that converge to this fixed point, such as the GTD algorithm \citep{sutton2009fast}, are hopeless in such cases. The ETD algorithm, on the other hand, has a bounded bias \emph{for all behavior policies}.

\begin{figure}
\includegraphics[width=0.5\textwidth]{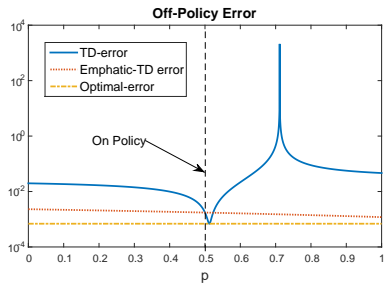}
\caption{\label{fig:off_policy_bias}Mean squared error in value function approximation for different behavior policies.}
\end{figure}
\section{The Bias-Variance Trade-Off of \ETDO}

From the results in Corollary \ref{corr:err}, it is clear that increasing the decay rate $\beta$ decreases the bias bound. Indeed, for the case $\beta=1$ we obtain the importance sampling TD algorithm \citep{precup2001off}, which is known to have a bias bound similar to on-policy TD. However, as recognized by \citet{precup2001off} and \citet{SuttonMW15}, the importance sampling ratio $F_t$ suffers from a high variance, which increases with $\beta$. The quantity $F_t$ is important as it appears as a multiplicative factor in the definition of the ETD learning rule, so its amplitude directly impacts the stability of the algorithm. In fact, the asymptotic variance of $F_t$ may be infinite, as we show in the following example:

\begin{Example}\label{example:inf_variance}
Consider the same MDP given in Example \ref{example:tightness}, only now the behavior policy chooses Left or Right with probability $0.5$, and the target policy chooses always Right. For \ETDO\ with $\beta\in [0,1)$, we have that when $S_t = \text{Left}$ then $F_t=1$ (since $\rho_{t-1} = 0$). When $S_t= \text{Right}$, $F_t$ may take several values depending on how many steps, $\tau(t)$, was the last transition from Left to Right, i.e.~$\tau(t)\eqdef\min \{i\geq 0: S_{t-i}=Left \}$. We can write this value as $F^{\tau(t)}$ where:
$$F^{\tau} \doteq \sum_{i=0}^{\tau} (2\beta)^i = \frac{(2\beta)^{\tau+1}-1}{2\beta-1},$$
if $2\beta\neq 1$. Let us assume that $2\beta > 1$ since interesting cases happen when $\beta$ is close to 1.

Let's compute $F_t$'s average over time: Following the stationary distribution of the behavior policy, $S_t=\text{Left}$ with probability $1/2$. Now, conditioned on $S_t=\text{Right}$ (which happens with probability $1/2$), we have $\tau(t)=i$ with probability $2^{-i-1}$. Thus the average (over time) value of $F_t$ is


\begin{equation*}
\E F_t = \frac{1}{2} \sum_{i=0}^{\infty} 2^{-i-1} F^i = \frac{\sum_i \beta^{i+1}-1}{2(2\beta-1)}  = \frac{1}{2(1-\beta)}.
\end{equation*}

Thus $F_t$ amplifies the TD update by a factor of $\frac{1}{2(1-\beta)}$ in average. Unfortunately, the actual values of the (random variable) $F_t$ does not concentrate around its expectation, and actually $F_t$ does not even have a finite variance. Indeed the average (over time) of $F_t^2$ is

\begin{equation*}
\E F_t^2 = \frac{1}{4} \sum_{i=0}^{\infty} 2^{-i} (F^i)^2 = \frac{\sum_{i} 2^{-i} \big( (2\beta)^{i+1}-1 \big)^2}{4(2\beta-1)^2}  = \infty,
\end{equation*}

as soon as $2\beta^2\geq 1$.
\end{Example}

So although \ETDO\ converges almost surely (as shown by \citealt{yu2015etd}), the variance of the estimate may be infinite, which suggests a prohibitively slow convergence rate.

In the following proposition we characterize the dependence of the variance of $F_t$ on $\beta$.
\begin{Proposition}\label{prop:variance}
Define the mismatch  matrix $\tilde{P}_{\mu,\pi}$ such that $[\tilde{P}_{\mu,\pi}]_{\bar{s} s} = \sum_a p(s|\bar{s}, \bar{a}) \frac{\pi^2(a | \bar{s})}{\mu(a | \bar{s})}$ and write $\alpha(\mu,\pi)$ the largest magnitude of its eigenvalues. Then for any $\beta < 1/\sqrt{\alpha(\mu,\pi)}$ the average variance of $F_t$ (conditioned on any state) is finite, and 
$$ \E_\mu \left[ \var[F_t|S_t = s] \right] \leq \frac{\beta^2}{1-\beta} \left( 2 + \frac{ (1 + \beta)\norm{\tilde{P}_{\mu,\pi}}_\infty}{1 - \beta^2 \norm{\tilde{P}_{\mu,\pi}}_\infty}\right), $$ 
where $\norm{\tilde{P}_{\mu,\pi}}_\infty$ is the $l_\infty$-induced norm which is the maximum absolute row sum of the matrix. 

\end{Proposition}

\begin{proof}(Partial) 
Following the same derivation that \citet{SuttonMW15} used to prove that $f(s)=d_\mu(s) \lim_{t\rightarrow \infty } \E [F_t | S_t = s]$, we have
\begin{equation*}
\begin{split}
	q(s) & \doteq d_\mu (s) \lim_{t\rightarrow \infty } \E [F^2_t | S_t = s] \\
	& = d_\mu (s) \lim_{t\rightarrow \infty } \E [ (1 + \rho_{t-1} \beta F_{t-1})^2 | S_t = s] \\
	&= d_\mu (s) \lim_{t\rightarrow \infty } \E [ 1 + 2\rho_{t-1} \beta F_{t-1} + \rho^2_{t-1} \beta^2 F^2_{t-1} | S_t = s].
\end{split}	
\end{equation*}
For the first summand, we get $d_\mu (s)$. For the second summand, we get:	
\begin{equation*}
	2\beta d_\mu (s) \lim_{t\rightarrow \infty } \E [ \rho_{t-1} F_{t-1} | S_t = s] = 2 \beta \sum_{\bar{s}} [P_\pi]_{\bar{s} s} f(\bar{s}).
\end{equation*}
The third summand equals
\begin{equation*}
\begin{split}
	& \beta^2 \sum_{\bar{s}, \bar{a}} d_\mu(\bar{s}) \mu(\bar{a} | \bar{s}) p(s|\bar{s}, \bar{a})\frac{\pi^2(\bar{a} | \bar{s})}{\mu^2(\bar{a} | \bar{s})}\lim_{t\rightarrow\infty} \E [F^2_{t-1}|S_{t-1} = \bar{s}] \\
	&= \beta^2 \sum_{\bar{s}, \bar{a}} p(s|\bar{s}, \bar{a}) \frac{\pi^2(\bar{a} | \bar{s})}{\mu(\bar{a} | \bar{s})} q(\bar{s}) 
	= \beta^2 \sum_{\bar{s}} [\tilde{P}_{\mu,\pi}]_{\bar{s} s} q(\bar{s}).
\end{split}	
\end{equation*}
Hence $q = d_\mu + 2\beta P^\top_\pi f + \beta^2 \tilde{P}^\top_{\mu,\pi} q$.
Thus for any $\beta < 1/\sqrt{\alpha(\mu,\pi)}$, all eigenvalues of the matrix $\beta^2 \tilde{P}^\top_{\mu,\pi}$ have magnitude smaller than  $1$, and the vector $q$ has finite components. The rest of the proof is very technical and is given in Lemma \ref{lem:var-bound} in the supplementary material. 

\end{proof}

Proposition \ref{prop:variance} and Corollary \ref{corr:err} show that the decay rate $\beta$ \emph{acts as an implicit trade-off parameter} between the bias and variance in ETD. For large $\beta$, we have a low bias but suffer from a high  variance (possibly infinite if $\beta \geq 1/\sqrt{\lambda(\mu,\pi)}$), and vice versa for small $\beta$. Notice that for the on-policy case, $\lambda(\mu,\pi)=1$ thus for any $\beta < 1$ the variance is finite.

Originally, \ETDO\ was introduced with $\beta=\gamma$, and from our perspective, it may be seen as a specific choice for the bias-variance trade-off. However, there is no intrinsic reason to choose $\beta=\gamma$, and other choices may be preferred in practice, depending on the nature of the problem. In the following numerical example, we investigate the bias-variance dependence on $\beta$, and show that the optimal $\beta$ in term of mean-squared error may be quite different from $\gamma$.

\subsection{Numerical Illustration}
We revisit the 2-state MDP described in Section \ref{ssec:example_Kolter1}, with $\gamma=0.9$,  $\epsilon=0.2$ and $p=0.95$. For these parameter settings, the error of standard TD is $42.55$ ($p$ was chosen to be close to a point of infinite bias for these parameters).

In Figure \ref{fig:bias_variance} we plot the mean-squared error $\norm{ \Phi \T \theta^* - V^\tpol}_{\tdist}$, where $\theta^*$ was obtained by running \ETDO\ with a step size $\alpha=0.001$ for $10,000$ iterations, and averaging the results over $10,000$ different runs.

First of all, note that for all $\beta$, the error is smaller by two orders of magnitude than that of standard TD. Thus, algorithms that converge to the standard TD fixed point such as GTD \cite{sutton2009fast} are significantly outperformed by \ETDO\ in this case. Second, note the dependence of the error on $\beta$, demonstrating the bias-variance trade-off discussed above. Finally, note that the minimal error is obtained for $\gamma=0.8$, and is considerably smaller than that of the original ETD with $\beta=\gamma=0.9$.

\begin{figure}
\includegraphics[width=0.5\textwidth]{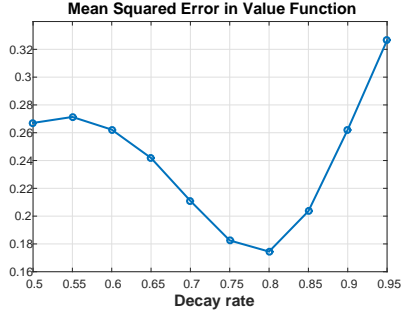}
\caption{\label{fig:bias_variance}Mean squared error in value function approximation for different decay rates $\beta$.}
\end{figure}

\section{Contraction Property for \ETD}

We now extend our results to incorporate eligibility traces, in the style of the ETD($\lambda$) algorithm \citep{SuttonMW15}, and show similar contraction properties and error bounds.

The \ETD\ algorithm iteratively updates the weight vector $\theta$ according to
\begin{equation*}
\begin{split}
\theta_{t+1} &:= \theta_t + \alpha (R_{t+1} + \gamma \theta_t \T \phi_{t+1} - \theta_t \T \phi_t) e_t \\
e_t &= \rho_t (\gamma \lambda e_{t-1} + M_t \phi_t ), \quad e_{-1} = 0 \\
M_t &= \lambda + (1-\lambda)F_t \\
F_t &= \beta \rho_{t-1} F_{t-1} + 1, \quad F_0 = 1,
\end{split}
\end{equation*}
where $e_t$ is the eligibility trace \citep{SuttonMW15}.
In this case, we define the emphatic weight vector $m$ by
\begin{equation}\label{eq:m}
m \T = \bdist\T (I - P^{\lambda, \beta})^{-1},
\end{equation}
where $P^{a, b}$ for some $a,b\in \R$ denotes the following matrix:
\begin{equation*}
\begin{split}
P^{a, b} &= I - (I-b a \Pmat)^{-1} (I-b \Pmat).
\end{split}
\end{equation*}
The Bellman operator for general $\lambda$ and $\gamma$ is given by:
\begin{equation*}
T^{(\lambda)} (V) = (I - \gamma \lambda \Pmat)^{-1}\Rvec + P^{\lambda, \gamma}V, \quad V\in \R^{|S|}.
\end{equation*}
For $\lambda=0$ we have $P^{\lambda, \beta}_\tpol=\beta P$, $P^{\lambda, \gamma}_\tpol=\gamma P$, and $m=f$ so we recover the definitions of \ETDO.

Recall that our goal is to estimate the value function $V^\tpol$. Thus, we would like to know how well the \ETD\ solution approximates $V^\tpol$. \citet{Mahmood2015ETD} show that, under suitable step-size conditions, ETD converges to some $\theta^*$ that is a solution of the \emph{projected fixed-point equation}:
\begin{equation*}
    \theta \T \Phi  = \Pi_m T^{(\lambda)} (\theta \T \Phi).
\end{equation*}
In their analysis, however, \citet{Mahmood2015ETD} did not show how well the solution $\Phi \T \theta^*$ approximates $V^\tpol$. Next, we establish that the projected Bellman operator $\Pi_m T^{(\lambda)}$ is a contraction. This result will then allow us to bound the error $\norm{ \Phi \T \theta^* - V^\tpol }_m$.

\begin{Theorem}\label{thm:two}
	$\Pi_m T^{(\lambda)}$ is an $\omega$-contraction with respect to the Euclidean $m$-weighted norm where:
	\begin{equation}
	\begin{split}
		\beta \geq \gamma : \quad \omega & = \sqrt{ \frac{\gamma^2 (1 + \lambda\beta )^2 (1-\lambda) }{\beta (1+\gamma \lambda)^2 (1-\lambda \beta)} } , \\
			\beta \leq \gamma : \quad \omega & = \sqrt{\frac{\gamma^2 (1-\beta \lambda ) (1-\lambda)}{\beta (1-\gamma \lambda)^2}} .
	\end{split}
	\end{equation}
\end{Theorem}

\begin{proof}(sketch)
	The proof is almost identical to the proof of Theorem $\ref{thm:one}$, only now we cannot apply Jensen's inequality directly, since the rows of $P^{\lambda, \beta}$ do not sum to $1$. However:
	\begin{equation*}
	P^{\lambda, \beta} \textbf{1} = \left( I - (I-\beta \lambda \Pmat)^{-1} (I-\beta \Pmat) \right) \textbf{1} = \zeta \textbf{1},
	\end{equation*}
	where $\zeta = \frac{\beta (1-\lambda)}{1-\lambda \beta}$. Notice that each entry of $P^{\lambda, \beta}$ is positive. Therefore $\frac{P^{\lambda, \beta}}{\zeta}$ will hold for Jensen's inequality.
	Let $M = diag(m)$, we have
	\begin{equation*}
	\begin{split}
	\norm{ v }^2_m - \frac{1}{\zeta} \norm{ P^{\lambda, \beta} v }^2_m &= v^\top M v - \zeta v^\top \frac{P^{\lambda, \beta}}{\zeta}^\top M \frac{P^{\lambda, \beta}}{\zeta} v  \\
	&\geq^{(a)} v^\top M v - \beta v^\top diag(m^\top \frac{P^{\lambda, \beta}}{\zeta})v \\
	& = v^\top [M -  diag(m^\top P^{\lambda, \beta})]v \\
	& = v^\top \left[diag \left(m^\top (I- P^{\lambda, \beta}) \right)  \right] v \\
	&=^{(b)} v^\top diag(d_\mu) v = \norm{ v }^2_{d_\mu},
	\end{split}
	\end{equation*}
	where (a) follows from the Jensen inequality and (b) from Equation (\ref{eq:m}). Therefore:
	\begin{equation*}
	\norm{ v }^2_m \geq \frac{1}{\zeta} \norm{ P^{\lambda, \beta} v } ^2_m + \norm{ v }^2_{d_\mu} \geq \frac{1}{\zeta} \norm{ P^{\lambda, \beta} v } ^2_m,
	\end{equation*}
	and:
	\begin{equation*}
	\begin{split}
	\norm{T^{(\lambda)} v_1 - T^{(\lambda)} v_2 }^2_m  &= \norm{ P^{\lambda,\gamma} (v_1 - v_2) } ^2_m  \\
	\text{(Case A: } \quad \beta \geq \gamma \text{)} & \leq \norm{\frac{\gamma (1+\beta \lambda)}{\beta ( 1+\gamma \lambda)}  P^{\lambda, \beta} (v_1 - v_2) } ^2_m \\
		& \leq \frac{\gamma^2 (1 + \lambda\beta )^2 (1-\lambda) }{\beta (1+\gamma \lambda)^2 (1-\lambda \beta)} \norm{ v_1 - v_2 } ^2_m, \\
	\text{(Case B: } \quad \beta \leq \gamma \text{)} & \leq \norm{\frac{\gamma (1-\beta \lambda)}{\beta ( 1-\gamma \lambda)}  P^{\lambda, \beta} (v_1 - v_2) } ^2_m \\
	& \leq \frac{\gamma^2 (1-\beta \lambda ) (1-\lambda)}{\beta (1-\gamma \lambda)^2} \norm{ v_1 - v_2 } ^2_m .
	\end{split}
	\end{equation*}
	The inequalities depending on the two cases originate from the fact that the two matrices $ P^{\lambda, \beta},  P^{\lambda, \gamma}$ are polynomials of the same matrix $P_\tpol$, and mathematical manipulation on the corresponding eigenvalues decomposition of $(v_1-v_2)$. The details are given in Lemma~\ref{lem:norm-ineq} of the supplementary material.

	Now, for a proper choice of $\beta$, the operator $T^{(\lambda)}$ is a contraction, and since $\Pi_m$ is a non-expansion in the $m$-weighted norm, $\Pi_m T^{(\lambda)}$ is a contraction as well.
\end{proof}

In Figure \ref{fig:contraction} we illustrate the dependence of the contraction moduli bound on $\lambda$ and $\beta$. In particular, for $\lambda \to 1$, the contraction modulus diminishes to 0. Thus, for large enough $\lambda$, a contraction can always be guaranteed (this can also be shown mathematically from the contraction results of Theorem \ref{thm:two}). We remark that a similar result for standard TD($\lambda$) was established by \citealt{yu2012least}. However, as is well-known \citep{Ber2012DynamicProgramming}, increasing $\lambda$ also increases the variance of the algorithm, and we therefore obtain a bias-variance trade-off in $\lambda$ as well as $\beta$.
Finally, note that for $\beta = \gamma$, the contraction modulus equals $\sqrt{\frac{\gamma(1-\lambda)}{1-\gamma\lambda}}$, and that for $\lambda=0$ the result is the same as in Theorem \ref{thm:one}.
\begin{figure}
	\includegraphics[width=0.5\textwidth]{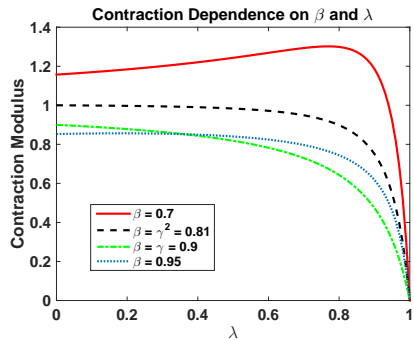}
	\caption{\label{fig:contraction}Contraction moduli of $\Pi_m T^{(\lambda)}$ for different $\beta$'s, as a function of the bootstrapping parameter $\lambda$. Notice that we see a steep decrease in the moduli only for $\lambda$ close to 1. }
	\end{figure}
	
%

\section{Conclusion}
In this work we unified several off-policy TD algorithms under the \ETD\ framework, which flexibly manages the bias and variance of the algorithm by controlling the decay-rate of the importance-sampling ratio. From this perspective, we showed that several different methods proposed in the literature are special instances of this bias-variance selection.

Our main contribution is an error analysis of \ETD\ that quantifies the bias-variance trade-off. In particular, we showed that the recently proposed ETD algorithm of \citet{SuttonMW15} has bounded bias for \emph{general} behavior and target policies, and that by controlling the decay-rate in the \ETD\ algorithm, an improved performance may be obtained by reducing the variance of the algorithm while still maintaining a reasonable bias.

Possible future extensions of our work includes finite-time bounds for off-policy \ETD, an error propagation analysis of off-policy \emph{policy improvement}, and solving the bias-variance trade-off adaptively from data.


\bibliographystyle{aaai}
\bibliography{ContractionBib}

%


\onecolumn
\newpage
\appendix
\section{Proof of Lemma 1}
Notice that $\kappa$ obtains non-negative values since $d_\mu(s), f(s) \geq 0$. Now, if there is a state $s$ visited by the target policy, but not the behavior policy, this means that $d_\mu(s) = 0$, and that there is some $t$ such that $[ d_\mu^\top P^t_\pi ](s) > 0$, and by definition $f(s) \geq [\beta^t d_\mu^\top P^t_\pi ](s)$, so we can get $\kappa = 0$.

Next, we prove the upper bound on $\kappa$. Notice that $f(s)\geq 0$, and that $\sum_s f(s) = 1/(1-\beta)$. Hence, if $d_\mu \neq (1-\beta) f$, then there must exist some $s$ such that $d_\mu(s) < (1-\beta) f(s)$ so $\kappa < 1-\beta$. Now, when $d_\mu = d_\pi$, by definition $d_\mu = (1-\beta)f$ and we obtain this upper bound.

\section{Technical Part of Proposition 1}
\begin{Lemma}\label{lem:var-bound}
The following is true:
\begin{equation*}
\sum_s d_\mu (s)  \lim_{t\rightarrow \infty } \var[F_t | S_t = s] \leq \frac{\beta^2}{1-\beta} \left( 2 + \frac{ (1 + \beta)\norm{\tilde{P}_{\mu,\pi}}_\infty}{1 - \beta^2 \norm{\tilde{P}_{\mu,\pi}}_\infty}\right).
\end{equation*}
\end{Lemma}
\begin{proof}
Notice that:
\begin{equation*}
f^\top = d_\mu^\top (I - \beta P_\pi)^{-1} \geq^{(cw)} d_\mu^\top + \beta d_\mu^\top P_\pi,
\end{equation*}
so:
\begin{equation*}
\begin{split}
\sum_s & d_\mu (s)  \lim_{t\rightarrow \infty } \var[F_t | S_t = s] = q^\top \textbf{1} - f^\top D_\mu^{-1} f \\
& \leq ^{(a)} d^\top_\mu \textbf{1} + 2\beta f^\top P_\pi \textbf{1} \\
& \quad + (d^\top_\mu + 2\beta f^\top P_\pi )\beta^2 \tilde{P}_{\mu,\pi}(I - \beta^2 \tilde{P}_{\mu,\pi} )^{-1}\textbf{1} \\
& \quad - (d_\mu + \beta P^\top_\pi d_\mu)^\top D_\mu^{-1}(d_\mu + \beta P^\top_\pi d_\mu) \\
& \leq ^{(b)} (1 + \frac{2\beta}{1-\beta}) - (1 + 2\beta) \\
& \quad + \norm{(d^\top_\mu + 2\beta f^\top P_\pi )\beta^2 \tilde{P}_{\mu,\pi}(I - \beta^2 \tilde{P}_{\mu,\pi} )^{-1}}_1 \\
& \leq ^{(c)} \frac{2\beta^2}{1-\beta} + \beta^2 \norm{d^\top_\mu + 2\beta f^\top P_\pi}_1 \norm{\tilde{P}_{\mu,\pi}(I - \beta^2 \tilde{P}_{\mu,\pi} )^{-1}}_\infty \\
& \leq ^{(d)} \frac{\beta^2}{1-\beta} \left( 2 + \frac{ (1 + \beta)\norm{\tilde{P}_{\mu,\pi}}_\infty}{1 - \beta^2 \norm{\tilde{P}_{\mu,\pi}}_\infty}\right).
\end{split}
\end{equation*}
Where (a) comes from the inequality on $f$, (b) also removes the negative summand $\beta^2 d^\top_\mu P_\pi D^{-1}_\mu P^\top_\pi d_\mu$, and swaps sum with $l_1$ norm (all coordinates are non-negative), (c) and (d) are from the sub-multiplicative property of induced norms (the $l_\infty$ norm originates from the transpose).
\end{proof}

\section{Norm Inequality between $\P^{\lambda,\beta}_\pi$ and $\P^{\lambda, \gamma}_\pi$}
\begin{Lemma}\label{lem:norm-ineq}
If $\beta \geq \gamma$:
\begin{equation}
 \norm{ P^{\lambda,\gamma}_\tpol v } ^2_m \leq  \norm{\frac{\gamma (1 + \beta \lambda)}{\beta ( 1 + \gamma \lambda)}  P^{\lambda, \beta}_\tpol v } ^2_m,
\end{equation}
and if $\beta \leq \gamma$:
\begin{equation}
 \norm{ P^{\lambda,\gamma}_\tpol v } ^2_m \leq  \norm{\frac{\gamma (1-\beta \lambda)}{\beta ( 1-\gamma \lambda)}  P^{\lambda, \beta}_\tpol v } ^2_m,
\end{equation}
\end{Lemma}
\begin{proof}
Mark the orthonormal eigenvectors w.r.t. $m$, and corresponding eigenvalues of $P_\tpol$ by $u_j, t_j$ respectively ($t_j$ may be a complex number, this decomposition exists over $\mathbb{C}$ almost surely). Notice that since $ P^{\lambda, \beta}_\tpol,  P^{\lambda, \gamma}_\tpol$ are polynomials of $P_\tpol$ they have the same eigenvectors, with the eigenvalues $l^\beta_j := \frac{\beta t_j (1-\lambda)}{1-\beta \lambda t_j}, l^\gamma_j :=\frac{\gamma t_j (1-\lambda)}{1-\gamma \lambda t_j}$ correspondingly. Hence, we can write the first norm as follows:
\begin{equation}
\begin{split}
 \norm{ P^{\lambda,\gamma}_\tpol v } ^2_m  &= \norm{ P^{\lambda,\gamma}_\tpol \sum_j <u_j, v> u_j } ^2_m \\
 &= \norm{  \sum_j < u_j, v > P^{\lambda,\gamma}_\tpol u_j } ^2_m \\
 &= \norm{  \sum_j <u_j, v> l^\gamma_j u_j } ^2_m \\
&= \sum_j \norm{   <u_j, v> l^\gamma_j u_j } ^2_m \\
 &= \sum_j  \left| <u_j, v> \right|^2 \left| l^\gamma_j \right|^2 \norm{   u_j } ^2_m.
\end{split}
\end{equation}
And similarly for $\beta$:
\begin{equation}
 \norm{ P^{\lambda,\beta}_\tpol v } ^2_m  =  \sum_j  \left| <u_j, v> \right|^2 \left| l^\beta_j \right|^2 \norm{   u_j } ^2_m.
\end{equation}
So if we can find a constant $\alpha$ such that:
\begin{equation}
\forall j: \quad \left| l^\gamma_j \right|^2 \leq \alpha^2 \left| l^\beta_j \right|^2,
\end{equation}
then could swap  $\norm{ P^{\lambda,\gamma}_\tpol v } ^2_m \leq  \norm{ \alpha P^{\lambda,\gamma}_\tpol v } ^2_m$. The expression we want to maximize is:
\begin{equation}
\begin{split}
\frac{\left| l^\gamma_j \right|^2}{\left| l^\beta_j \right|^2} &= \frac{\gamma^2 (1-\beta \lambda t_j)(1-\beta \lambda t^*_j)}{\beta^2 (1-\gamma \lambda t_j)(1-\gamma \lambda t^*_j)} \\
&= \frac{\gamma^2 (1-\beta \lambda t_j - \beta \lambda t^*_j + \beta^2 \lambda^2 \left| t_j \right|^2)}{\beta^2 (1-\gamma \lambda t_j - \gamma \lambda t^*_j + \gamma^2 \lambda^2 \left| t_j \right|^2)}.
\end{split}
\end{equation}
Taking the derivative with respect to $Re(t_j), Im(t_j)$, shows that there are no extrema points inside the ball $\left| t_j \right| \leq 1$ (we know the eigenvalues are inside this ball since they belong to a stochastic matrix), which means we can look at the boundary of this ball $\left| t_j \right| = 1$ to find the maximum value. Since now we get dependence only on $Re(t_j)$, the maximum must be on $t_j= \pm 1$:
\begin{equation}
\max_{t:\left| t_j \right| \leq 1} \frac{\left| l^\gamma_j \right|^2}{\left| l^\beta_j \right|^2} = \frac{\gamma^2 (1 \pm \beta \lambda )^2}{\beta^2 (1 \pm \gamma \lambda)^2},
\end{equation}
where when $\beta \geq \gamma$ the plus is larger and vice versa.
\end{proof}

\end{document}